\documentclass[journal,twoside,web]{ieeecolor}
\usepackage{lcsys}
\usepackage{cite}
\usepackage{amsmath,amssymb,amsfonts,amsthm}
\usepackage{algorithmic}
\usepackage{graphicx}
\usepackage{textcomp}
\usepackage[euler]{textgreek}
\usepackage[hidelinks]{hyperref}
\usepackage{cleveref}
\usepackage{multirow}
\usepackage{centernot}
\usepackage{caption}
\usepackage{subcaption}
\let\labelindent\relax
\usepackage{tabularray}
\usepackage{enumitem,kantlipsum}

\theoremstyle{definition}
\newtheorem{theorem}{Theorem}
\newtheorem{lemma}{Lemma}

\newtheorem{assumption}{Assumption}
\newtheorem{remark}{Remark}

\newcommand{\revised}[1]{\textcolor{red}{#1}}
    
\begin{document}

\title{Chance-Constrained Trajectory Planning with Multimodal Environmental Uncertainty}
\author{
Kai Ren, Heejin Ahn, and Maryam Kamgarpour	\thanks{ Ren is with the ECE department, University of British Columbia, Vancouver, BC, Canada. Ahn is with the ECE department, Seoul National University, South Korea. Kamgarpour is with the SYCAMORE Lab, École Polytechnique Fédérale de Lausanne (EPFL), Switzerland. Contact e-mail: {\tt\small renkai@ece.ubc.ca}.}
}
\maketitle
\thispagestyle{empty} 

\begin{abstract} 
We tackle safe trajectory planning under Gau-ssian mixture model (GMM) uncertainty. Specifically, we use a GMM to model the multimodal behaviors of obstacles' uncertain states. Then, we develop a mixed-integer conic approximation to the chance-constrained trajectory planning problem with deterministic linear systems and polyhedral obstacles.  When the GMM moments are estimated via finite samples, we develop a tight concentration bound to ensure the chance constraint with a desired confidence. Moreover, to limit the amount of constraint violation, we develop a Conditional Value-at-Risk (CVaR) approach corresponding to the chance constraints and derive a tractable approximation for 
known and estimated GMM moments. We verify our methods with state-of-the-art trajectory prediction algorithms and autonomous driving datasets. 
\end{abstract}


\section{Introduction}
\IEEEPARstart{T}{he} safe operation of autonomous systems, such as self-driving cars, robots,
in uncertain environment is a central challenge in autonomy. For example, a self-driving car encounters other autonomous or human driven cars and needs to plan its trajectory to avoid collision despite the uncertainties in the future motion of these other vehicles. Our work addresses safe trajectory planning under environment uncertainty.

A chance-constrained program (CCP) is a common formulation to ensure safety in uncertain environments. Instead of enforcing the constraint for all uncertainty realizations, which often lead to conservative decisions, the chance-constrained formulation tolerates a small probability of constraint violation. Although CCPs are intractable due to the non-convex chance constraints, they can be tractably addressed for certain classes of probability distributions, such as Gaussian distribution with known moments \cite{blackmore2006, Blackmore2011,jha-jar18, Calafiore2006CC}. In practice, the exact Gaussian moments may be unknown and thus estimated from samples, e.g., sensor observations. In such cases, the CCP also has a tractable reformulation that guarantees safety with high confidence \cite{Lefkopoulos2021}. When the uncertainty's distribution cannot be captured by Gaussian, \cite{Han2022} considered the non-Gaussian uncertainty by mapping the probabilistic constraints to constraints on the moments of the state probability distribution.

In trajectory prediction problems, such as those arising in autonomous driving \cite{Salzmann2020CoRR, Rhinehart2019CoRR}, the probability distributions over the future positions of the road agents (vehicles, humans) are multimodal. The reason is that the agents often have different high-level intents in complex and interactive environments. For example, a vehicle may go straight or turn at an intersection.
Incorporating this additional information on the distribution can improve the performance of the planning algorithms.  

To account for the multimodal distribution in trajectory planning, Ahn \textit{et al}. \cite{Ahn_2022} developed a method to cluster the distribution samples corresponding to the different modes and then used a scenario-based approach \cite{Calafiore2006Scenario, campi2009} to provide safety guarantee with finite number of samples. The scenario-based approach does not   incorporate the prior knowledge on the true distribution. In particular, state-of-the-art trajectory prediction algorithms \cite{Salzmann2020CoRR, Rhinehart2019CoRR}, use a  Gaussian mixture model (GMM) to represent the distribution of the future vehicles' positions. Motivated by the prevalence of GMM in trajectory predictions, our work here focuses on developing a framework for safe trajectory planning under GMM.

Chance constraints under a GMM uncertainty can be equivalently formulated as a deterministic second-order cone constraint \cite{Hu2022}, similar to chance constraints with unimodal Gaussian distribution. As for applications, Yang \textit{et al}. \cite{yang2020} exploited the GMM to model the wind power uncertainty and solved a chance-constrained unit commitment problem. However, to our knowledge, chance constraints with GMM have not been applied to trajectory planning problems. 

Motivated by trajectory planning, we formulate a chance-constrained problem under GMM. To  manage the severity of constraint violation in safety-critical applications, we also consider the conditional value-at-risk (CVaR) \cite{Parys2016,Majumdar2017,Hakobyan2019, Barbosa2021} approximation of chance constraints. The CVaR constraint implies the original chance constraint and furthermore, limits the amount of violation in case of constraint violation. In both chance and CVaR constraints, we consider the cases in which the GMM moments are known or estimated from samples. 

The main contributions of this paper are as follows.
\begin{itemize} 
    \item We  formulate the chance-constrained trajectory planning problem under GMM, derive a deterministic formulation of the problem and its conservative approximation via CVaR in the case of known moments.
    
    \item For the case of moments learned through samples, we derive a tight GMM moment concentration bound. We use this bound to probabilistically guarantee the feasibility of the chance constraints and their CVaR approximations.
    
    \item We test our methods on a real-world autonomous driving \textit{nuScenes} dataset \cite{nuScenes}. We show that modelling the uncertain parameter's distribution with GMM induces less conservative motion than unimodal Gaussian modelling. The CVaR-constrained planner limits the constraint violation amount while ensuring safety with high probability.
\end{itemize}

The rest of the paper is organized as follows. In Section \ref{section:GMMCCandRobustification} we formulate the chance constraints with GMM uncertain parameter, and present the deterministic reformulations and approximations via CVaR for known and learned moments. Section \ref{section:motionplanning} reformulates a chance-constrained trajectory planning problem and its CVaR approximation as computationally tractable problems and exploit the derived moment concentration bound to guarantee safety. Section \ref{section:casestudy} demonstrates our methods in  real-world autonomous driving case studies.

\textit{Notation:} A Gaussian distribution with mean $\mu$ and covariance matrix $\Sigma$ is denoted as $\mathcal{N}(\mu, \Sigma)$. By $\text{\textPsi}^{-1}(\cdot)$ and $\Phi(\cdot)$, we denote the inverse cumulative distribution function and probability density function of the standard Gaussian distribution $\mathcal{N}(0, 1)$. By $(\cdot)_+$, we denote the operator $\max \{ \cdot \; , 0 \}$. We denote a set of consecutive integers $\{a, a+1, \dots,b\}$ by $\mathbb{Z}_{a:b}$. The matrix $A$ being positive definite is denoted as $A \succ 0$. We denote the conjunction by $\bigwedge$ and the disjunction by $\bigvee$. 

\section{Chance Constraint with GMM Uncertainty} \label{section:GMMCCandRobustification}
The trajectory planning problem aims to prevent the ego vehicle (EV) from colliding with other vehicles (OVs). Due to the uncertain future positions of OVs, we enforce the probability of constraint violation to be bounded by a prescribed threshold $\epsilon \in (0, 0.5)$. Let $x \in \mathbb{R}^{n_x}$ encode the state of the EV and $\delta$ be the uncertain parameter that follows a distribution $p_*$. Let $C(x, \delta)$ be the constraint function and $C(x, \delta) \leq 0$ represent the satisfaction of the constraint. The chance constraint can be written as follows.
\begin{equation} \label{eq:singlecc}
    \mathbb{P}_{\delta \sim p_*}(C(x, \delta) \leq 0) \geq 1-\epsilon.
\end{equation}

\subsection{Linear Constraint}
Let $\delta \in \mathbb{R}^{n}$ encode the uncertain location of an edge of the OV. The dimension of $\delta$ is $n=n_x+1$. Let $\Tilde{x} = [x^{\intercal} \;\; 1 ]^{\intercal}$. The EV being away from an edge of the OV can be represented as a linear constraint: ${\delta}^{\intercal}\Tilde{x} \leq 0$. A 2-dimensional space example, i.e. $n_x = 2$, is shown in Fig.~\ref{fig:linearconstraint}. Motivated by the constraints arising in EV planning, we consider the following assumption: \vspace{-2mm}

\begin{assumption} \label{assumption:guassianConstraint}
     The constraint function is in linear product form: $C(x, \delta) = \delta^{\intercal} \Tilde{x} \in \mathbb{R}$.
\end{assumption}

\subsection{Chance and CVaR Constraints under GMM}
Let $z \in \mathbb{R}^{n}$ be a random variable in the same dimension of $\delta$. The GMM with $K$ modes can be represented as follows.
\begin{equation*}
    p_*(z) = \sum_{k=1}^{K} \pi_k p_k(z), \; \;\;\;\; \sum_{k=1}^{K} \pi_k = 1, 
\end{equation*}
where each mode $p_k(z)$ is a unimodal Gaussian distribution $\mathcal{N}(\mu_k, \Sigma_k)$. For $p_*$ being GMM, the chance constraint \eqref{eq:singlecc} can be equivalently decomposed as follows \cite{Hu2022}.
\begin{subequations}
    \begin{alignat} {2}
     & \mathbb{P}_{\delta \sim p_k}(\delta^{\top}\Tilde{x} \leq 0) \geq 1-\epsilon_k, \;\;\;\;\; \forall k \in \mathbb{Z}_{1:K}, \label{constraint:CCsingleGMMmode} \\
     & \sum_{k=1}^{K} \pi_k \epsilon_k = \epsilon. \label{subconstraint:riskForgmm}
    \end{alignat}
\end{subequations}

\begin{figure}[!b]
    \centering \vspace{-1mm}
    \includegraphics[width=6.5cm]{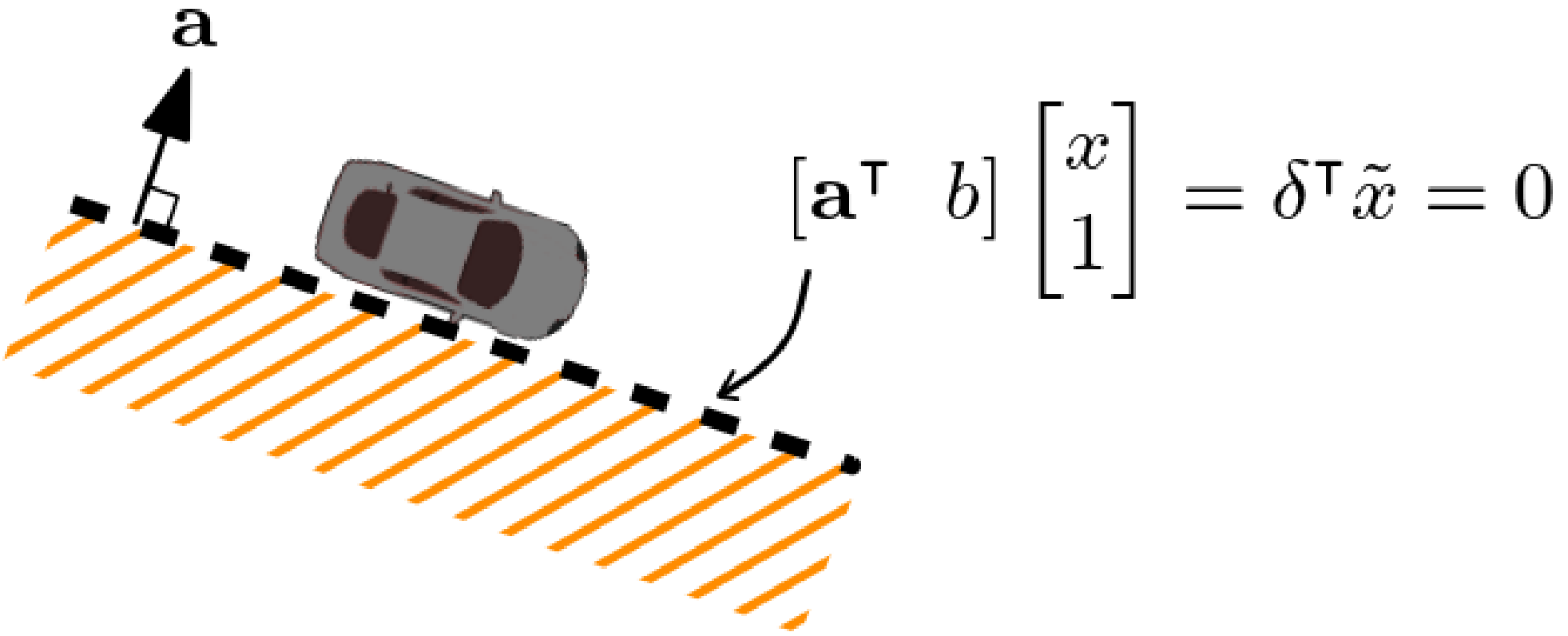}
    \caption{Linear constraint in 2-dimensional space}
    \label{fig:linearconstraint}
\end{figure}

As chance constraints cannot distinguish severe and mild constraint violations, we also consider a risk-aware approximation of the chance constraints in \eqref{constraint:CCsingleGMMmode} based on Conditional Value-at-risk (CVaR), where for the $k^{th}$ mode
\begin{equation} \label{eq:singlecvar}
    \underset{\delta \sim p_k}{\text{CVaR}_{\epsilon_k}}\left(\delta^{\top}\Tilde{x}\right) := 
    \inf_{\alpha>0} \left\{ \frac{1}{\epsilon_k} \mathbb{E}[(\delta^{\top}\Tilde{x} + \alpha)_+] - \alpha \right\} \leq 0.
\end{equation}

CVaR evaluates the expected constraint violation amount among the $\epsilon$-worst constraint values. CVaR constraint is a conservative approximation of a chance constraint, i.e. $\eqref{eq:singlecvar}\implies\mathbb{P}_{\delta \sim p_k}(\delta^{\top}\Tilde{x} \leq 0) \geq 1-\epsilon_k$. Hence, we can approximate \eqref{constraint:CCsingleGMMmode} with the following CVaR constraints:
\begin{equation} \label{constraint:CVaRsingleGMMmode}
    \underset{\delta \sim p_k}{\text{CVaR}_{\epsilon_k}}(\delta^{\top}\Tilde{x}) \leq 0, \;\;\;\;\; \forall k \in \mathbb{Z}_{1:K}.
\end{equation}

In the rest of the paper, we will approximate \eqref{eq:singlecc} via \eqref{eq:singlecc} $\Leftrightarrow$ \eqref{constraint:CCsingleGMMmode} $\wedge$ \eqref{subconstraint:riskForgmm} $\Leftarrow$ \eqref{constraint:CVaRsingleGMMmode} $\wedge$ \eqref{subconstraint:riskForgmm} for the risk-aware approximation.

\subsection{Deterministic reformulation \& Moment robustification} 
With Assumption \ref{assumption:guassianConstraint} and GMM uncertain parameter, we can deterministically reformulate the chance and CVaR constraints.
\begin{lemma} \label{lemma:gmmreformulation}
When the GMM moments $\mu_k$ and $\Sigma_k$ are known for all $k \in \mathbb{Z}_{1:K}$, the chance constraints \eqref{constraint:CCsingleGMMmode} and the CVaR approximation \eqref{constraint:CVaRsingleGMMmode} can be equivalently reformulated as the following second-order cone constraint:
\begin{equation} \label{constraint:gmmCC&CVaRDeterministic}
    \Gamma_k \sqrt{\Tilde{x}^{\intercal} \Sigma_k \Tilde{x}} + \mu_k^{\intercal} \Tilde{x} \leq 0, \;\;\;\;\; \forall k \in \mathbb{Z}_{1:K},
\end{equation}
where $\Gamma_k=\text{\textPsi}^{-1}(1-\epsilon_k)$ for chance constraint \cite{Calafiore2006CC} and $\Gamma_k=\Phi\left(\text{\textPsi}^{-1}(1-\epsilon_k)\right)/\epsilon_k$ for the CVaR approximation \cite{norton2019}.
\end{lemma}

In practice, the exact GMM moments $(\mu_k, \Sigma_k)$ are often not available. Hence, we also consider the case when the GMM moments are estimated from $N_s$ samples of $\delta$. 
\begin{remark}
Given a set of samples $\{\delta_1, \ldots, \delta_{N_s}\}$, in order to estimate $(\mu_k, \Sigma_k)$ for each GMM mode, we first need to determine the mode that each sample belongs to. Existing tools such as the expectation maximization method \cite{bishop2006} can split data into $K$ modes and determine which mode each sample belongs to. For trajectory planning scenarios, the state-of-the-art trajectory prediction model \cite{Salzmann2020CoRR} uses latent variables to encode the multimodal intents corresponding to each predicted sample trajectory, which categorizes the sample trajectories into $K$ modes. The existence of tools for categorizing samples into GMM modes motivate the following assumption. The relaxation of this assumption remains as a future work.

\begin{assumption} \label{assumption:gmmclustering}
    For a set of samples $\{\delta_1, \dots, \delta_{N_s}\}$, the mode $k$ where each sample belongs and the weight of each GMM mode $\pi_k$ are known for all $k \in \mathbb{Z}_{1:K}$.
\end{assumption}
\end{remark}

Now we have samples from each GMM mode. We denote $N_k$ as the sample size of the $k^{th}$ GMM mode. We can estimate $\mu_k$ and $\Sigma_k$ based on the $N_k$ samples. We denote the estimated moments as $\hat{\mu}_{k}$ and $\hat{\Sigma}_{k}$. We will use a moment concentration bound to robustify \eqref{constraint:gmmCC&CVaRDeterministic} against the moment estimation error.
\begin{theorem} \label{theorem:gmmrobustification}
Consider the case in which the moments of each Gaussian mode are estimated from $N_k$ samples. If the following inequality holds, then \eqref{constraint:gmmCC&CVaRDeterministic} holds with probability at least $1-2\beta$, where $\beta \in (0, 1)$ is a safety tolerance.
    \begin{equation} \label{constraint:gmmMRAconstraint}
        \Gamma_{k} \sqrt{(1+r_{2,k})(\Tilde{x}^{\intercal}\hat{\Sigma}_{k}\Tilde{x})} + r_{1, k} (x) + \hat{\mu}_{k}^{\intercal} \Tilde{x} \leq 0, \;\; \forall k \in \mathbb{Z}_{1:K}.
    \end{equation}
where $r_{1,k}$ and $r_{2,k}$ are defined as
    \begin{subequations}
        \begin{alignat}{2}
           & r_{1, k}(x) := \sqrt{\frac{T^{2}_{1, N_k-1}(1-\beta)}{N_k}} \sqrt{\Tilde{x}^{\intercal}\hat{\Sigma}\Tilde{x}}, \label{eq:1dMeanCbound} \\
           & r_{2, k} := \max \left\{|1-\frac{N_k-1}{\chi^{2}_{N_k-1, 1-\beta/2}}|, |1-\frac{N_k-1}{\chi^{2}_{N_k-1, \beta/2}}| \right\}.  \label{eq:covarianceCbound}
        \end{alignat}
    \end{subequations}
Here, $T^{2}_{a, b}(p)$ denotes the $p$-th quantile of the Hotelling’s T-squared distribution with parameters $a$ and $b$, and $\chi^{2}_{k, p}$ is the $p$-th quantile of the $\chi^2$-distribution with $k$ degrees of freedom. 
\end{theorem}
\begin{proof}
Consider the case when $\delta$ lies in the $k^{th}$ mode of the GMM. When the moments $\mu_k$ and $\Sigma_k$ are estimated from $N_k$ samples of $\delta$, with probability $1-\beta$ \cite{Lefkopoulos2021}, 
    \begin{subequations}
        \begin{alignat} {2}
        & \|\hat{\mu}_k - \mu_k\|_2 \leq R_{1, k} := \sqrt{\frac{T^{2}_{n, N_k-1}(1-\beta)}{N_k \lambda_{min}(\hat{\Sigma}^{-1})}}, \label{eq:ndMeanCbound} \\
        &  \|\Tilde{x}^{\intercal}\Sigma_k\Tilde{x} \| \leq (1+r_{2, k}) \cdot  \|\Tilde{x}^{\intercal}\hat{\Sigma}_k\Tilde{x} \|. \label{eq:covarianceRobustification}
        \end{alignat}
    \end{subequations}
        
    To ensure feasibility of the exact chance constraint reformulation \eqref{constraint:gmmCC&CVaRDeterministic}, we need to robustify the constraint by compensating the gaps between exact $(\mu_k, \Sigma_k)$ and estimated $(\hat{\mu}_{k}, \hat{\Sigma}_{k})$ as shown in \eqref{eq:ndMeanCbound} and \eqref{eq:covarianceRobustification}. Based on Assumption \ref{assumption:guassianConstraint}, the constraint function is $C(x, \delta) = \delta^\intercal\Tilde{x} \sim \mathcal{N}(\mu_k^{\intercal} \Tilde{x}, \sqrt{\Tilde{x}^{\intercal} \Sigma_k \Tilde{x}}^2)$. Hence, we can directly robustify the gap between the true $\mu_k^{\intercal} \Tilde{x} \in \mathbb{R}$ and estimated $\hat{\mu}_k^{\intercal} \Tilde{x}\in \mathbb{R}$ constraint value. To do this, we substitute ($\mu_k$, $\Sigma_k$) with ($\mu_k^{\intercal} \Tilde{x}$, $\sqrt{\Tilde{x}^{\intercal} \Sigma_k \Tilde{x}}^2$) and $n=1$ in \eqref{eq:ndMeanCbound}, which gives us $|\hat{\mu}_k^{\intercal} \Tilde{x} - \mu_k^{\intercal} \Tilde{x}| \leq r_{1,k} (x)$ with probability $1-\beta$. Based on this result and \eqref{eq:covarianceRobustification}, we can robustify constraint \eqref{constraint:gmmCC&CVaRDeterministic} by \eqref{constraint:gmmMRAconstraint} when the moments are estimated from samples.
    
    The uncertain parameter $\delta$ lies in the $k^{th}$ GMM mode with a probability $\pi_k$. When $\mu_k$ and $\Sigma_k$ are estimated from samples, a solution satisfying the $k^{th}$ constraint in \eqref{constraint:gmmMRAconstraint} satisfies the $k^{th}$ constraint in \eqref{constraint:gmmCC&CVaRDeterministic} with probability $(1-\beta)^2 > 1-2\beta$. Consider all $K$ modes, a feasible solution of \eqref{constraint:gmmMRAconstraint} is feasible to \eqref{constraint:gmmCC&CVaRDeterministic} with probability at least $\sum_{k=1}^{K} [\pi_k (1-2\beta)] = 1-2\beta$.
\end{proof}

Previous work \cite{Lefkopoulos2021} used \eqref{eq:ndMeanCbound} to robustify the mean estimation error. However, it did not consider the linear constraint structure arising from the trajectory planning problems. We observe that \eqref{eq:1dMeanCbound} provides a tighter robustification on the mean estimation error. Later, we will show with real-world trajectory planning examples that the robustified chance and CVaR-constrained planners yield infeasible solutions with the bound \eqref{eq:ndMeanCbound}, while the new bound \eqref{eq:1dMeanCbound} enables feasible solutions.

\begin{remark}
The mean concentration bound in \eqref{eq:1dMeanCbound} provides a tighter upper bound for $\mu_k^{\intercal} \Tilde{x}$ than that given in \eqref{eq:ndMeanCbound} for the following reason. For brevity, we omit the subscript $k$ for the rest of this remark. The concentration bound in \eqref{eq:ndMeanCbound} provides an upper bound on $\mu^\intercal \Tilde{x}$ as follows, with probability $1-\beta$.
    \begin{equation} \label{eq:ndMeanRobustification}
        \begin{aligned}
            \mu^{\intercal}\Tilde{x} & \leq \hat{\mu}^{\intercal}\Tilde{x} + R_{1}  \|\Tilde{x}\|_2 = \hat{\mu}^{\intercal}\Tilde{x} + C_1 h_1(x) \\
            & :=  \hat{\mu}^{\intercal}\Tilde{x} + \sqrt{\frac{T^{2}_{n, N_k-1}(1-\beta)}{N_k}} \sqrt{\frac{\Tilde{x}^{\intercal}\Tilde{x}}{\lambda_{min}(\hat{\Sigma}^{-1})}}.
        \end{aligned}
    \end{equation}
    With the concentration bound in \eqref{eq:1dMeanCbound} we can bound $\mu^\intercal \Tilde{x}$ from above as follows, with probability $1-\beta$.
    \begin{equation} \label{eq:1dMeanRobustification}
        \begin{aligned}
            \mu^\intercal \Tilde{x} & \leq \hat{\mu}^{\intercal} \Tilde{x} + r_{1}(x) = \hat{\mu}^{\intercal}\Tilde{x} + C_2 h_2(x)\\
            & := \hat{\mu}^{\intercal}\Tilde{x} + \sqrt{\frac{T^{2}_{1, N_k-1}(1-\beta)}{N_k}} \sqrt{\Tilde{x}^{\intercal}\hat{\Sigma}\Tilde{x}}.
        \end{aligned}
    \end{equation}
    We empirically observed that when $n \geq 1$, $C_1 \geq C_2$. Fig. \ref{fig:boundConstantComparison} shows a comparison of $C_1$ and $C_2$ for different $\beta, N_k$ and $n$. \vspace{1mm}
    
    \begin{figure}[ht] 
        \centering
        \includegraphics[width=0.5\textwidth]{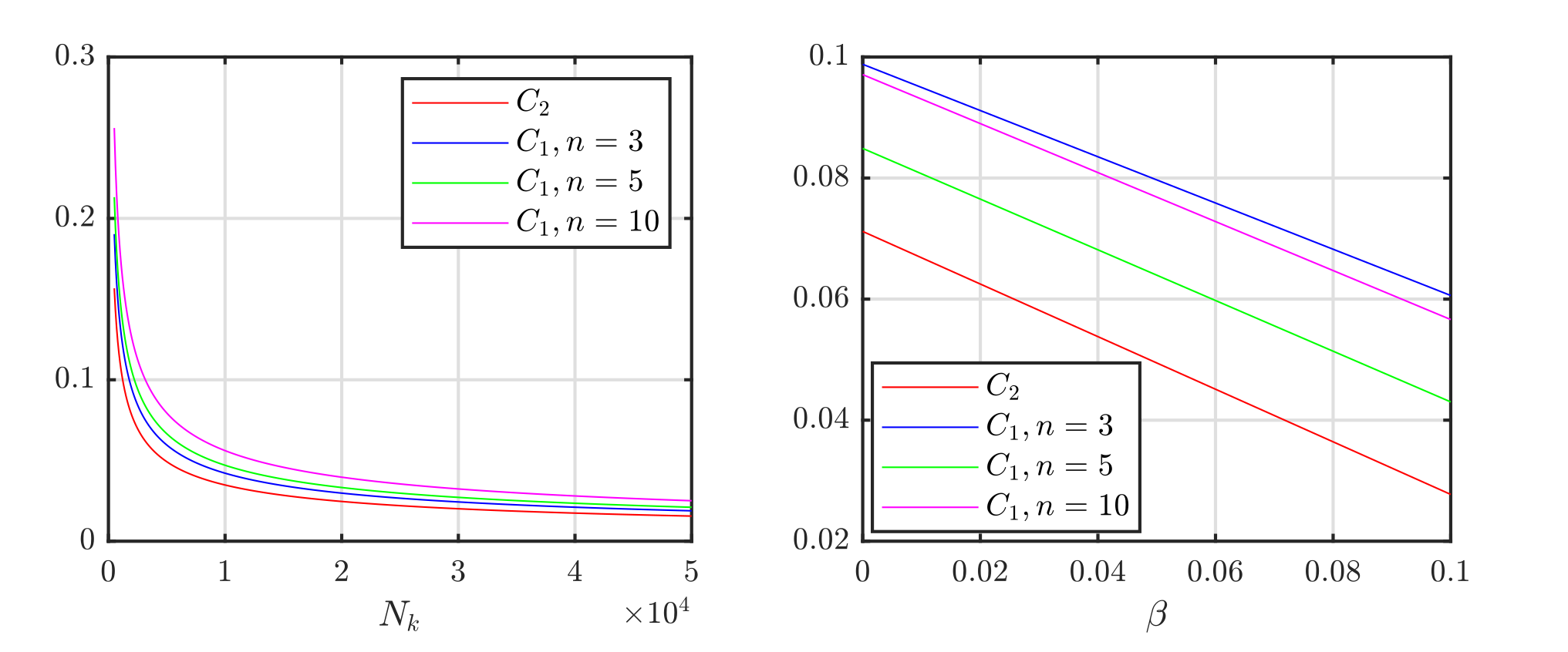} \vspace{-3.5mm}
        \caption{Comparison of $C_1$, $C_2$ for fixed $\beta = 10^{-3}$ and varying $N_k$ (left), fixed $N_k = 5000$ and varying $\beta$ (right).}
        \label{fig:boundConstantComparison}
    \end{figure}
    
   Now, as $\hat{\Sigma}$ is symmetric and $\hat{\Sigma} \succ 0$, we have:
    \small{\begin{equation*}
        \begin{aligned}
        \frac{\Tilde{x}^\intercal\Tilde{x}}{\lambda_{min}(\hat{\Sigma}^{-1})} =\lambda_{max}(\hat{\Sigma}) \cdot \Tilde{x}^\intercal\Tilde{x} \geq \Tilde{x}^\intercal \hat{\Sigma} \Tilde{x}.
        \end{aligned}
    \end{equation*}}
    Thus, $h_1(x) \geq h_2(x)$. We observe $C_1 h_1(x) \geq C_2 h_2(x)$, and the upper bound in \eqref{eq:1dMeanRobustification} appears to be tighter than that in \eqref{eq:ndMeanRobustification}.
\end{remark} 

\section{Chance-Constrained Trajectory Planning} \label{section:motionplanning}

We consider trajectory planning scenarios, in which the objective is to minimize a cost function (e.g., fuel/distance to the destination) while avoiding collision with any other vehicles (OVs) throughout the planning horizon. We consider the ego vehicle (EV) as a linear discrete-time system
\begin{equation} \label{eq:systemeq}
    x_{t+1} = A_t x_t + B_t u_t,
\end{equation}
where $x_t \in \mathcal{X} \in \mathbb{R}^{n_x} $ and $u_t \in \mathcal{U} \in \mathbb{R}^{n_u}$ are the state and input at time $t$, and $A_t \in \mathbb{R}^{n_x \times n_x}$ and $B_t \in \mathbb{R}^{n_x \times n_u}$ are system's dynamics matrices. The control inputs are constrained to a time-invariant convex set $\mathcal{U}$. With an initial state $x_0$ and a planning horizon $N$, the trajectory planner generates an input sequence $\textbf{u} = (u_0, u_1, \dots, u_{N-1})$. Given the initial state $x_0$ and the state transition function \eqref{eq:systemeq}, the input sequence $\textbf{u}$ would lead to a state trajectory $\textbf{x} = (x_0, x_1, \dots, x_{N})$ that should be free from collisions with the OVs. Note that we use a linear time-varying model, which is commonly used to approximate nonlinear systems via linearization (e.g., \cite{falcone2007}). 

Considering the uncertain behavior of the OV, the trajectory planning problem can be formulated as the following finite-horizon CCP:
    \begin{subequations} \label{problem:originalTP}
    \begin{alignat} {2}
    & \underset{\textbf{u}}{\text{min}}
    & \quad & f(x_0, \textbf{u}) \\
    & \text{s.t.}
    & \quad & \mathbb{P}_{\delta_{ij}^t \sim p_{ij*}^t} (\bigwedge_{t=1}^{T} \bigwedge_{j=1}^{J}  \bigvee_{i=1}^{I_j} {\delta_{ij}^t}^{\intercal} \Tilde{x}_t \leq 0) \geq 1-\epsilon. \label{constraint:originalconstraint} \\ 
    & &\quad& \text{\textbf{x}, \textbf{u} satisfy \eqref{eq:systemeq} with initial state } x_0. \label{constraint:initialState}
    \end{alignat}
    \end{subequations}
        
The uncertain parameter $\delta_{ij}^t$ corresponds to the $j^{th}$ OV's $i^{th}$ face at time $t$. The constraint enforces the EV to be away from at least one of the $I_j$ edges of the $j^{th}$ OV. To overcome the intractability of disjunction and conjunctions in chance constraint \eqref{constraint:originalconstraint}, we exploit the Big-M method \cite{Schouwenaars2001} and the Boole's inequality \cite{CaseBerg:01} to conservatively approximate \eqref{constraint:originalconstraint} as
    \begin{subequations} \label{constraint:allToSingleCC}
    \begin{alignat} {2}
        & \bigwedge_{t=1}^{T} \bigwedge_{j=1}^{J}  \bigwedge_{i=1}^{I_j} \mathbb{P}_{\delta_{ij}^t \sim p_{ij*}^t} ({\delta_{ij}^t}^{\intercal} \Tilde{x}_t + M z_{ij}^{t} \leq 0) \geq 1-\epsilon_{ij}^{t}, \label{MPconstraint:allConjunctionCC} \\
        & \sum_{i=1}^{I_j} z_{ij}^{t} = 1, \;\; z_{ij}^{t} \in \{0, 1\}, \label{MPconstraint:singleBinary} \\
        & \epsilon_{ij}^t = \epsilon/(TJ), \label{MPconstraint:riskConstraint}
    \end{alignat}
    \end{subequations}
where $z_{ij}^{t}$ is a binary variable to enforce one of the disjunctive constraints, i.e. the EV is away from one edge of the OV. For simplicity, the risk bound $\epsilon$ is uniformly assigned among OVs and time steps \cite{Lefkopoulos2021}. Other heuristics \cite{jha-jar18} or optimization-based \cite{Blackmore2011,blackmore2009} approaches can also be incorporated. 

Past works \cite{blackmore2006, Blackmore2011} considered $p_{ij*}^t$ as a Gaussian distribution, which cannot model the multimodal behaviors of the OVs. Thus, we extend $p_{ij*}^t$ to GMM with $K_j$ modes, where each mode is a Gaussian denoted as $p_{ijk}^t = \mathcal{N}(\mu_{ijk}^{t}, \Sigma_{ijk}^{t})$.

In the state-of-the-art trajectory prediction model \cite{Salzmann2020CoRR}, each prediction of future OV trajectories incorporates a latent variable that encodes the multimodal intents (i.e. going straight or turning). This provides the mode each sample trajectory belongs to. The number of the latent variables and the probability corresponding to each latent variable provides us the number of modes and the weight of each mode. The number and values of the latent variables remain the same throughout the prediction horizon, which motivate the following assumption:
    \begin{assumption} \label{assumption:clusternum}
        The number of modes $K_j$ and the weights $\{\pi_1, \dots, \pi_{K_j}\}$ for $p_{ij*}^t$ remain the same for all $t \in \mathbb{Z}_{0:T-1}$ and for all faces $i \in \mathbb{Z}_{1:I_j}$ of the $j^{th}$ OV.
    \end{assumption}
       
Based on Assumption \ref{assumption:clusternum}, \eqref{MPconstraint:allConjunctionCC} and \eqref{MPconstraint:singleBinary} can be further reformulated as follows.
    \begin{subequations} \label{constraint:gmmsingleTP}
        \begin{alignat} {2}
            &  \mathbb{P}_{\delta_{ij}^t \sim p_{ijk}^{t}} ( {\delta_{ijk}^t}^{\intercal} \Tilde{x}_t + M z_{ijk}^{t} \leq 0) \geq 1-\epsilon_{ijk}^{t}, \label{MPconstraint:gmmsingleCC}\\
            & \sum_{k=1}^{K_j} \pi_k \epsilon_{ijk}^{t} = \epsilon_{ij}^{t}, \;\;\; \sum_{k=1}^{K_j} \pi_k = 1 \label{MPconstraint:gmmSingleWeight} \\
            & \sum_{i=1}^{I_j} z_{ijk}^{t} = 1, \;\;\; z_{ijk}^{t} \in \{0, 1\}. \label{MPconstraint:gmmSingleBinary}
        \end{alignat}
    \end{subequations}

The constraint \eqref{MPconstraint:gmmsingleCC} need to be satisfied for all $i, j, k$ and $t$ throughout the planning horizon. For GMM uncertain parameter, we need to assign the risk bound $\epsilon_{ij}^{t}$ to enforce \eqref{MPconstraint:gmmSingleWeight}. A simple method is the uniform risk allocation (URA), where $\epsilon_{ijk}^{t} = \epsilon_{ij}^{t}$ for all $k \in \mathbb{Z}_{1:K_j}$
\footnote{An optimal risk allocation (ORA) assigns a higher $\epsilon_k$ to a GMM mode $k$ that poses higher threats to the constraint. One can use a Branch and Bound (B\&B) method \cite{Land1960} to get ORA \cite{Hu2022}, but it is computationally heavy and disallows online application. We applied B\&B for ORA in our trajectory planning simulations, and observed that the ORA improves minimally compared to the URA. When the URA planner is infeasible, it is also infeasible with the ORA.}. For the simulations in the paper, we allocate the risk bound based on the URA.
        
\subsection {Deterministic reformulation \& Moment robustification}
        
We now reformulate the chance constraint \eqref{MPconstraint:gmmsingleCC} and its CVaR approximation as deterministic constraints based on Lemma \ref{lemma:gmmreformulation} and Theorem \ref{theorem:gmmrobustification}. 
        
\begin{theorem}\label{theorem:single_cc}[Safe trajectory planning with GMM]
\begin{enumerate} [leftmargin=14pt]
    \item \label{subtheorem:knownMoments}
Given the moments ($\mu_{ijk}^{t}, \Sigma_{ijk}^{t}$) of the uncertain parameters for all $t, j, i \text{ and } k$ throughout the planning horizon, problem \eqref{problem:originalTP} and the CVaR approximation can be conservatively reformulated as the following mixed-integer second-order cone program: 
    \begin{subequations} \label{MPproblem:gmmDeterministic}
    \begin{alignat} {2}
    & \underset{\textbf{u}, \; z_{ijk}^t}{\text{min}}
    & \quad & f(x_0, \textbf{u}) \\
    & \text{s.t.}
    & & \Gamma_{ijk}^{t} \sqrt{ \Tilde{x}_t^{\intercal}{\Sigma_{ijk}^{t}} \Tilde{x}_t} + \Tilde{x}^{\intercal}\mu_{ijk}^{t} \leq Mz_{ijk}^{t}, \;\;\; \forall t,\; j,\; i,\; k. \label{MPconstraint:gmmDeterministic}\\ 
    & &\quad& \eqref{constraint:initialState},\; \eqref{MPconstraint:riskConstraint},\; \eqref{MPconstraint:gmmSingleWeight},\; \eqref{MPconstraint:gmmSingleBinary}, \label{MPconstraint:gmmDeterministicAllothers}
    \end{alignat}
    \end{subequations}
with $\Gamma_{ijk}^{t} = \text{\textPsi}^{-1}(1-\epsilon_{ijk}^{t})$ for the chance constraint, and $\Gamma_{ijk}^{t}=\Phi[\text{\textPsi}^{-1}(1-\epsilon_{ijk}^{t})]/\epsilon_{ijk}^{t}$ for the CVaR constraint. A feasible solution of \eqref{MPproblem:gmmDeterministic} is also feasible to \eqref{problem:originalTP}.

\item  Consider the case in which the moments are estimated from samples. We denote the estimated moments as ($\hat{\mu}_{ijk}^{t}, \hat{\Sigma}_{ijk}^{t}$). A solution that satisfies the following  and \eqref{MPconstraint:gmmDeterministicAllothers} is feasible to \eqref{problem:originalTP} with probability $1-2\beta TJ$.
    \begin{equation}
    \begin{aligned}
        & \Gamma_{ijk}^{t}\sqrt{(1+r_{2, ijk}^t)(\Tilde{x}_t^{\intercal}\hat{\Sigma}_{ijk}^{t}\Tilde{x}_t)} + r_{1, ijk}^t (x_t) \\ 
        & \;\;\;\;\;\;\;\;\;\;\; + \Tilde{x}^{\intercal}\hat{\mu}_{ijk}^{t} \leq Mz_{ijk}^{t}, \;\;\; \forall t, j, i \text{ and } k. \label{MPconstraint:gmmMRAdeterministic} 
    \end{aligned}
    \end{equation}
    Here, $r_{1, ijk}^t (x_t)$ and $r_{2, ijk}^t$ are defined in \eqref{eq:1dMeanCbound} and \eqref{eq:covarianceCbound}, with $N_k$ denoting the sample size of $\delta_{ijk}^{t}$. Also, $\Gamma_{ijk}^{t}$ is defined for the chance and CVaR constraints as in \ref{subtheorem:knownMoments}).
    \end{enumerate}
    \end{theorem}
    
    \begin{proof}
        When the GMM moments are known, the uncertain parameter $\delta_{ijk}^{t}$ conforms to $\mathcal{N}(\mu_{ijk}^{t}, \Sigma_{ijk}^{t})$ for each GMM mode. Based on Lemma \ref{lemma:gmmreformulation}, each single chance constraint in \eqref{MPconstraint:gmmsingleCC} can be reformulated as \eqref{MPconstraint:gmmDeterministic}. Thus, a solution that satisfies \eqref{MPconstraint:gmmDeterministic} and \eqref{MPconstraint:gmmDeterministicAllothers} conservatively satisfies \eqref{constraint:originalconstraint} because \eqref{constraint:originalconstraint} $\Leftarrow$ \eqref{constraint:allToSingleCC} $\Leftrightarrow$ \eqref{constraint:gmmsingleTP}$\wedge$\eqref{MPconstraint:riskConstraint} $\Leftrightarrow$  \eqref{MPconstraint:gmmDeterministic}$\wedge$\eqref{MPconstraint:gmmDeterministicAllothers}. When the GMM moments are estimated from samples, approximating \eqref{MPconstraint:gmmDeterministic} with \eqref{MPconstraint:gmmMRAdeterministic} is a direct application of Theorem \ref{theorem:gmmrobustification}.
        
        From Theorem \ref{theorem:gmmrobustification}, for each $i, j$ and $t$, a solution that satisfies \eqref{MPconstraint:gmmMRAdeterministic} is feasible to \eqref{MPconstraint:gmmsingleCC} with probability $1-2\beta$. If the constraints are satisfied jointly for all $TJ$ constraints (i.e. for all OVs at all time steps), the solution is feasible to \eqref{problem:originalTP} with a probability of at least $(1-2\beta)^{TJ} > 1-2\beta TJ$ \cite{Lefkopoulos2021}.
    \end{proof}
        
Problem \eqref{MPproblem:gmmDeterministic} and its robustified approximation for sample-estimated moments (i.e. when approximating \eqref{MPconstraint:gmmDeterministic} with \eqref{MPconstraint:gmmMRAdeterministic}) are mixed-integer second-order cone programs, which are computationally tractable and can be efficiently solved via off-the-shelf optimization solvers, such as CPLEX \cite{cplex2009v12}.

\section{Simulations and Results} \label{section:casestudy}
 
In this section, we test our methods on real-world trajectory planning examples.\footnote{The code used for the simulations is available at \href{https://github.com/renkai99/LCSS-control}{\texttt{LCSS-control}}.} We compare the performances of our methods by modelling the uncertain parameter's distribution with Gaussian uni-model (GUM) and GMM. All computations were conducted on an Intel i7 CPU at 2.60 GHz with 8 GB of memory using YALMIP \cite{Lofberg2004} and CPLEX \cite{cplex2009v12}.

\begin{figure}[!t] 
    \centering
    \includegraphics[width=0.49\textwidth]{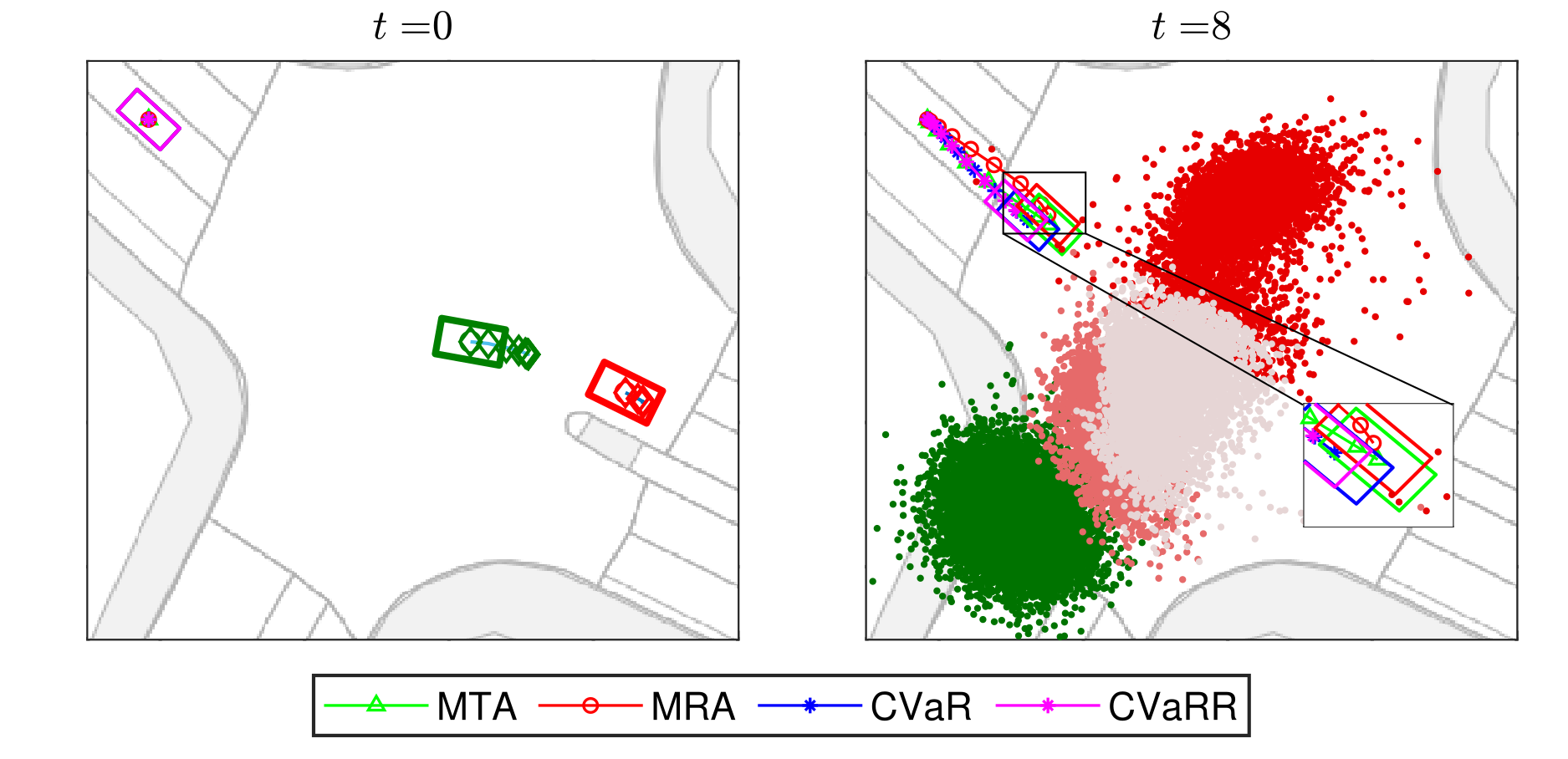}
    \vspace{-3.5mm}
    \caption{EV Trajectories in scene 1 with MTA, MRA, CVaR and CVaRR trajectory planners.}
    \label{fig:nuScenes1} \vspace{-1mm}
\end{figure}

\begin{table}[!b]
    \caption{Scene 1 Planning Results} \vspace{-2mm}
    \begin{center} 
    \begin{tabular}{| c | c | c | c | c |}
        \hline
        & \textbf{Method} & \textbf{Cost} & \textbf{VRate} & \textbf{$\mathbb{E}\{\text{VAmt}\}$} (m) \\ \hline
        \multirow{5}{*}{\shortstack[c]{Multimodal\\Modelling}}  
        & MTA & -11.04 & 0.045\% & 0.0448\\ \cline{2-5}
        & MRA & -10.02 & 0.018\% & 0.0377\\ \cline{2-5}
        & CVaR & -9.61 & 0.015\% & 0.0401\\ \cline{2-5}
        & CVaRR & -8.52 & 0.004\% & 0.0097 \\ \cline{2-5}
        & Scenario \cite{Ahn_2022} & -1.62 & 0 & 0 \\ \hline
        \rule{0pt}{15pt} \shortstack{Unimodal\\Modelling} & \shortstack{\vspace{1mm}All methods} & \shortstack{\vspace{1mm}$\infty$} & \shortstack{\vspace{1mm} \revised{-}} & \shortstack{\vspace{1mm} \revised{-}} \\ \hline
    \end{tabular}
    \end{center}
    \label{tab:nuScene1} 
\end{table}
    
We consider the same problem as case study 2 in \cite{Ahn_2022}, which applies the trajectory prediction neural network \textit{Trajectron++} \cite{Salzmann2020CoRR} on a real-world autonomous driving \textit{nuScenes} dataset \cite{nuScenes}. The \textit{Trajectron++} model predicts $N_s$ sample future trajectories of the OVs with a discrete latent variable encoding the high-level intents corresponding to each sample trajectory. We model the distribution of the uncertain parameter with GUM and GMM respectively. For GMM modelling, each latent variable corresponds to a GMM mode. At each time step, we estimate the moments of the uncertain parameter based on the $N_s$ samples. For both schemes, we first plan EV motions based on \eqref{MPproblem:gmmDeterministic} assuming the sample estimated moments are accurate. The chance-constrained formulation and the CVaR approximation are named as Moment Trust Approach (MTA) and CVaR approach (CVaR) respectively. Then we plan EV motions based on Theorem \ref{theorem:single_cc} with moment robustification \eqref{MPconstraint:gmmMRAdeterministic}. The chance-constrained formulation and the CVaR approximation are named as Moment Robust Approach (MRA) and CVaR Robust (CVaRR) approach respectively.

Two cross-intersection scenarios are investigated. Fig. \ref{fig:nuScenes1} and Fig. \ref{fig:nuScenes2} show the initial configuration of the two scenes at $t=0$. In scene 1, the magenta EV enters the intersection and two OVs (red and green) come from the opposite direction. The future trajectories of the red OV exhibit three modes (red, dark pink, and light pink). In scene 2, two OVs (red and green) enters the intersection from the opposite and the same directions of the magenta EV respectively. The future trajectories of the red OV exhibit two modes (red and dark pink). The EV is modelled as a double-integrator and the safe distance to the centroid of the EV is a half-diagonal length of the EV. The OVs are considered as rectangular obstacles. An OV rectangle represents an enlarged size of the OV to include a safe distance of 0.1 meters. Thus, the EV should stay outside the 0.1 meters protective distance. For all scenarios, we define the risk bound $\epsilon = 0.05$ and safety tolerance $\beta = 10^{-3}$. The performance of the planners are evaluated on 3 criterion: a cost value that we intend to minimize (smaller costs mean the EV makes more progress in its longitudinal direction and has less lateral displacement and velocity), a rate of constraint violation (denoted as \textbf{VRate}) and an expected violation amount (i.e. the average distance of the EV crashing into OV, which is denoted as \textbf{$\mathbb{E}\{\text{VAmt}\}$}). The violation rate and expected violation amount are evaluated based on $10^5$ new predictions from the \textit{Trajectron++} model. The results for scene 1 and scene 2 are shown in Table \ref{tab:nuScene1} and Table \ref{tab:nuScene2} respectively. The terminal positions of the EV at the end of the planning horizon $t=8$ in scene 1 and scene 2 are shown in Fig. \ref{fig:nuScenes1} and Fig. \ref{fig:nuScenes2} respectively.

\begin{figure}[!t] 
    \centering
    \includegraphics[width=0.49\textwidth]{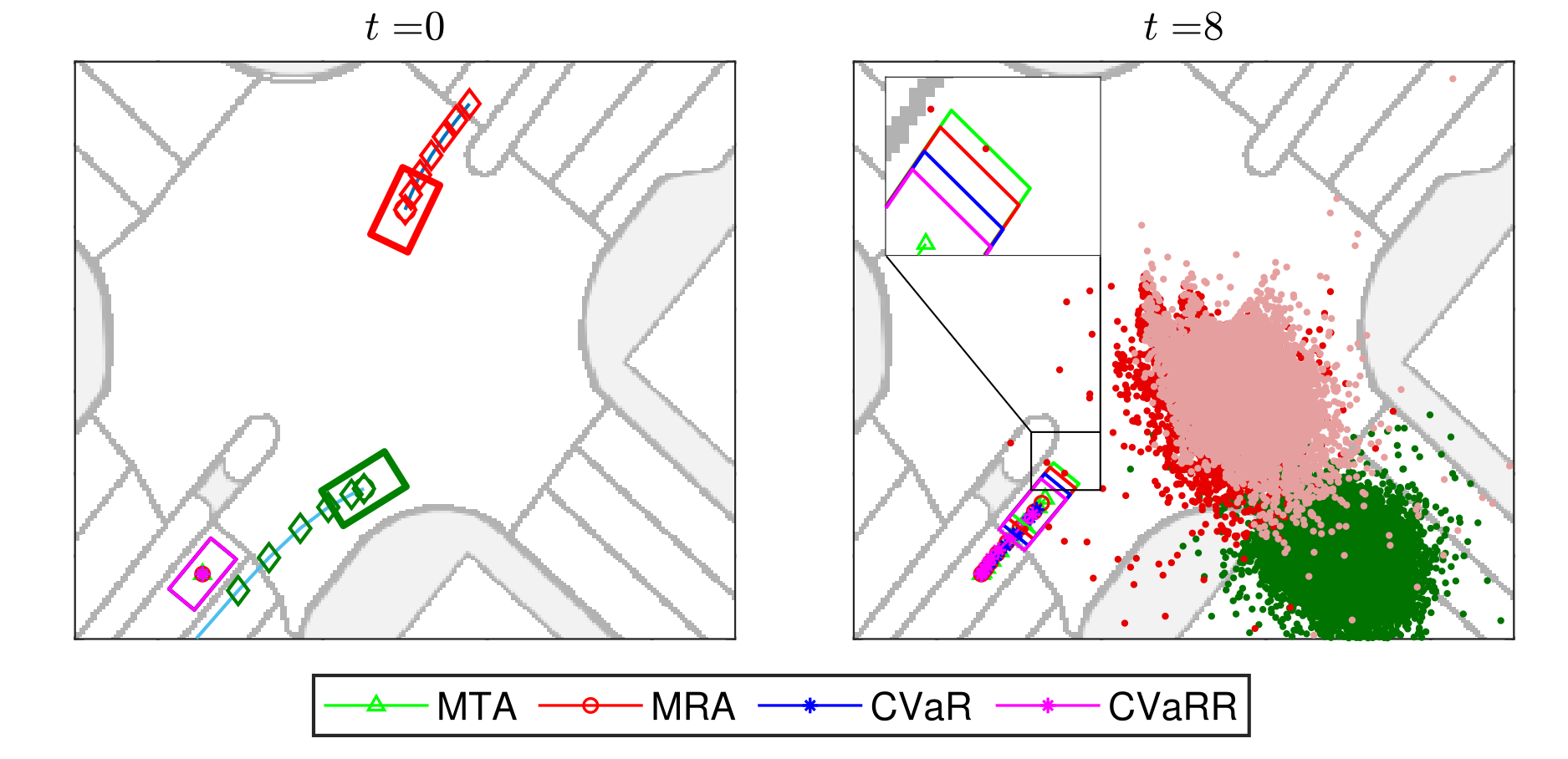}
    \vspace{-3.9mm}
    \caption{EV Trajectories in scene 2 with MTA, MRA, CVaR and CVaRR trajectory planners.}
    \label{fig:nuScenes2} \vspace{-1mm}
\end{figure}

\begin{table}[b] \vspace{-1mm}
    \caption{Scene 2 Planning Results}
    \begin{center} \vspace{-2mm}
    \begin{tabular}{| c | c | c | c | c |} 
        \hline
        & \textbf{Method} & \textbf{Cost} & \textbf{VRate} & \textbf{$\mathbb{E}\{\text{VAmt}\}$} (m) \\ \hline
        \multirow{5}{*}{\shortstack[c]{Multimodal\\Modelling}}  
        & MTA & -5.99 & 0.334\% & 0.0374\\ \cline{2-5}
        & MRA & -5.63 & 0.224\% & 0.0343\\ \cline{2-5}
        & CVaR & -5.13 & 0.131\% & 0.0361 \\ \cline{2-5}
        & CVaRR & -4.75 & 0.085\% & 0.0283 \\ \cline{2-5}
        & Scenario \cite{Ahn_2022} & -1.47 & 0 & 0 \\ \hline
        \rule{0pt}{15pt} \shortstack{Unimodal\\Modelling} & \shortstack{\vspace{1mm}All methods} & \shortstack{\vspace{1mm}$\infty$} & \shortstack{\vspace{1mm} \revised{-}} & \shortstack{\vspace{1mm} \revised{-}} \\ \hline
    \end{tabular}
    \end{center}
    \label{tab:nuScene2}
\end{table} 
        
\textbf{Optimality:} By modelling the uncertain parameter's distribution as unimodal Gaussian, all the planners become infeasible. On the other hands, the GMM modelling enables feasible solutions for all the planners. As CVaR conservatively approximates chance constraint, CVaR (CVaRR) planner is more conservative than MTA (MRA) as expected. This can be seen from the progress of the EV in the longitudinal direction at the end of the planning horizon $t=8$, which is shown by the magnified insets in Fig. \ref{fig:nuScenes1} and Fig. \ref{fig:nuScenes2}.

\textbf{Risk level:} The probability of constraint violation for all planners are below the predefined threshold $\epsilon=0.05$. 

\textbf{Computational Time:} Within 1.5 s, the MTA, MRA and CVaR planners can output an EV trajectory for the 4 s horizon. The CVaRR planner takes 4.25 s to yield a trajectory. Reducing the computational time for CVaRR remains as a future work. 

\textbf{Expected Violation:} The CVaRR planner yields the least expected amount of violation when the constraints are violated. Fig. \ref{fig:worstcasecollision} shows a single case where the output trajectories of all planners collide with the OV. When the constraint is violated, the CVaRR planner has the minimum amount of violation.

\textbf{Comparison to past works:} We implement MRA and CVaRR planners with the mean robustification \eqref{eq:ndMeanRobustification} proposed in \cite{Lefkopoulos2021}, which produced infeasible solutions even with GMM modelling. However, our approach \eqref{eq:1dMeanRobustification} enables feasible motions. We also compare our results with the scenario-based approach proposed in \cite{Ahn_2022}. Our methods produces trajectories that yield significantly smaller cost values.

\begin{figure}[!t] 
    \centering
    \includegraphics[width=0.46\textwidth]{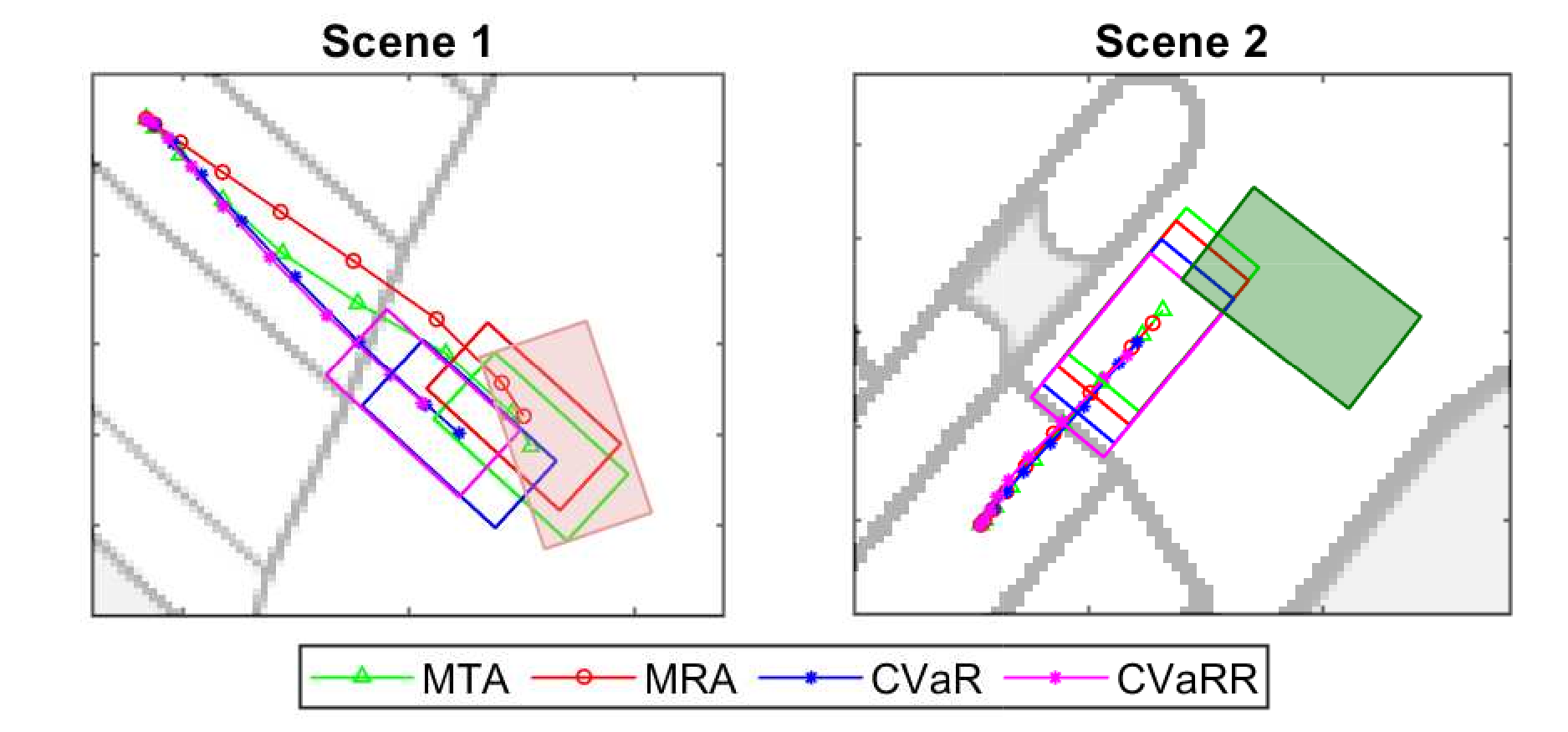} \vspace{-1mm}
    \caption{Worst-case collisions for MTA, MRA, CVaR and CVaRR trajectory planners.}
    \label{fig:worstcasecollision} 
\end{figure}

\section{Conclusion}
We presented a trajectory planner that considers multimodal uncertain obstacles and time-varying linear systems. We modeled the distribution of the uncertain parameter by GMM, motivated by the state-of-the-art trajectory prediction algorithms. This enables deterministic reformulations of chance and CVaR constraints. When the GMM moments are estimated from samples, we derived a tight bound on the estimation error to robustify the constraints. The  robustification probabilistically guarantees the feasibility of chance and CVaR constraints with a finite number of samples. We demonstrated our approaches in real-world autonomous driving examples and showed that modelling the uncertain parameter's distribution with GMM yields less conservative trajectories than modelling the uncertainty's distribution with unimodal Gaussian or using scenario-based approach. Our methods also ensured that the constraint violation rate is bounded by the threshold. Also, the CVaR-constrained planner can limit the expected violation amount when the constraints are violated. Currently, we are extending our work to conduct closed-loop trajectory planning.
    
\bibliographystyle{IEEEtran}
\bibliography{main}

\end{document}